\documentclass{article}
\usepackage{kr}

\usepackage{times}
\usepackage{soul}
\usepackage{url}
\usepackage[hidelinks]{hyperref}
\usepackage[utf8]{inputenc}
\usepackage[small]{caption}
\usepackage{graphicx}
\usepackage{amsmath}
\usepackage{amsthm}
\usepackage{booktabs}
\usepackage{algorithm}
\usepackage{algorithmic}
\usepackage{mdframed} 
\definecolor{light-gray}{gray}{0.95}

\newtheorem{theorem}{Theorem}
\newtheorem{lemma}{Lemma}[section]

\makeatletter
\newenvironment{innerproof}[1][\proofname]
  {\par\normalfont \topsep6\p@ \@plus6\p@\relax
  \trivlist
  \item[\hskip\labelsep\itshape#1\@addpunct{.}]\ignorespaces}
  {\endtrivlist\@endpefalse}
\makeatother

\newtheorem{definition}[]{Definition}
\usepackage{natbib}

\usepackage{listings}
\DeclareCaptionStyle{ruled}{labelfont=normalfont,labelsep=colon,strut=off} 
\floatstyle{ruled}
\newfloat{listing}{tb}{lst}{}
\floatname{listing}{Listing}

\usepackage{multirow}

\newcommand{\asp}[1]{\mbox{$\mathtt{#1}$}}

\DeclareMathOperator{\codeif}{\mathtt{:-}}
\DeclareMathOperator{\hash}{\mathtt{\#}}
\DeclareMathOperator{\naf}{\;\mathtt{not}\;}

\title{A Unifying Framework for Learning Argumentation Semantics}

\author{%
Zlatina Mileva$^1$\and
Antonis Bikakis$^2$\and
Fabio Aurelio D'Asaro$^{3}$\and
Mark Law$^{4}$ \\
Alessandra Russo$^1$ \\
\affiliations
$^1$Imperial College London, UK\\
$^2$University College London, UK\\
$^3$University of Verona, Italy\\
$^4$ILASP Limited, UK\\
\emails
zlatina.mileva19@imperial.ac.uk,
a.bikakis@ucl.ac.uk,
fabioaurelio.dasaro@univr.it,\\
mark@ilasp.com, alessandra.russo@imperial.ac.uk
}

\begin{document}

\maketitle

\begin{abstract}
Argumentation is a very active research field of Artificial Intelligence concerned with the representation and evaluation of arguments used in dialogues between humans and/or artificial agents. Acceptability semantics of formal argumentation systems define the criteria for the acceptance or rejection of arguments. Several software systems, known as argumentation solvers, have been developed to compute the accepted/rejected arguments using such criteria. These include systems that learn to identify the accepted arguments using non-interpretable methods. In this paper, we present a novel framework that uses an Inductive Logic Programming approach to learn the acceptability semantics for several abstract and structured argumentation frameworks in an interpretable way. Through an empirical evaluation we show that our framework outperforms existing argumentation solvers, thus opening up new future research directions in the area of formal argumentation and human-machine dialogues.
\end{abstract}

\section{Introduction}
Argumentation is a typical human activity that involves the use of arguments to resolve a conflict of opinion or to draw conclusions based on evidence that may be incomplete or contradictory. Understanding how humans reason with arguments has many real-world applications in Artificial Intelligence (AI), see, e.g., \citep{abghprstv:17}. This is why argumentation has become a well-established area within the field of knowledge representation and reasoning. The study of argumentation in AI has mainly focused on developing formal models of argumentation and methods for supporting argumentative tasks such as identifying arguments, evaluating arguments and drawing conclusions from them. Formal argumentation models fall within two general categories: abstract models, which treat arguments as atomic entities without an internal structure, with the \textit{Abstract Argumentation Framework} \cite[AAF;][]{pm:95} being the most prominent example; and structured models, which make explicit the premises and the claims of arguments as well as the relationship between them, with the \textit{Assumption-Based Argumentation Framework} \cite[ABA;][]{t:13} being a characteristic example. The acceptability semantics of each such model is a formal specification of the criteria for accepting or rejecting an argument in a given framework.

Argumentation solvers compute acceptable arguments in a given framework. These can usually be of two types, e.g., \textit{exact} solvers are manually designed to implement any given known semantics. One exemplary solver is ASPARTIX \citep{egw:10}, based on \textit{Answer Set Programming} \cite[ASP;][]{l:19}. Another type of solvers includes \textit{approximate} solvers, which use Machine Learning (ML) or Deep Learning (DL) techniques to learn the argumentation semantics from data. One such example is \citep{cb:20}, which uses Deep Learning to learn argumentation semantics for AAFs. Deep Learning techniques have proved to be very effective -- still, they lack transparency, thus leaving the tasks of learning argumentation semantics in an interpretable way an open research problem that we aim to tackle in the present paper.

This work aims to address this problem by proposing a framework that \textit{learns} argumentation semantics for a variety of argumentation frameworks in an efficient and \textit{interpretable} way. The proposed framework uses an \textit{Inductive Logic Programming} (ILP) approach called \textit{Learning from Answer Sets} \cite[LAS;][]{LawRB19} to learn Answer Set Programs that capture the acceptability semantics of several argumentation frameworks from a (small) set of examples. In this work we first show that our learning method is \textit{complete} with respect to some standard semantics, i.e., it can learn given known semantics if it is given appropriate examples in its training set. To this aim, we manually engineer argumentation frameworks from which these semantics can be learned and prove that, in fact, the corresponding learned Answer Set Programs are equivalent to the manually engineered ASP encodings of acceptability semantics used in ASPARTIX. We also show that, in many cases, they exhibit better time performance compared to ASPARIX. On the other hand, we show our method does not need as much training data with respect to state-of-the-art approximate solvers.

The paper is organized as follows. In Section \ref{sec:argumentation}, we provide the background on Argumentation and Learning Answer Set Programs. In Section \ref{sec:unifying} we introduce our framework and show how it can be used to learn ASP encodings for admissible, complete, grounded, preferred and stable semantics of four argumentation frameworks: AAF, Bipolar Argumentation Frameworks, Value-based Argumentation Frameworks and ABA. In the Evaluation section, we present the results of the empirical evaluation that compares our framewoork with state-of-the-art argumentation solvers for the stable and complete semanrics. Specifically, our framework learns semantics, which, at the time of inference, are much faster than the manually engineered ASPARTIX ASP encodings. Moreover, this is achieved in a more data-efficient manner compared to the large datasets required by approximate DL methods. Therefore, in a sense, our framework provides the best of both worlds for the considered cases. The paper concludes by discussing possible future directions for this research.

Finally, note that the datasets used for learning, the learning methods, the encoding and the benchmark code for reproducibility are all freely available at \href{https://github.com/dasaro/ArgLAS}{\texttt{https://github.com/dasaro/ArgLAS}}.


\section{Background}
\label{sec:argumentation}
\subsection{Argumentation}
The study of argumentation in AI has mostly focused on formal models of argumentation and computational methods for supporting argumentative tasks, such as identifying arguments, evaluating arguments, etc. One of the most influential models of argumentation is \emph{Abstract Argumentation Frameworks} \cite[AAF;][]{pm:95}. Its main characteristic is that it models arguments as atomic entities (without an internal structure), and the acceptability of an argument depends only on the attacks it receives from other arguments. An AAF is defined as a pair $\langle Arg,att \rangle$ consisting of a set of arguments $Arg$ and a binary attack relation $att$ on this set; for any two arguments $a,b \in Arg$, we say that $a$ attacks $b$ when $(a,b) 
\in att$. Acceptability semantics for AAFs are defined in terms of extensions, i.e., sets of arguments that we can reasonably accept.
An extension of an AAF, $S \subseteq Arg$ is called: \emph{conflict-free} iff it contains no arguments attacking each other; \emph{admissible} iff it is conflict-free and defends all its elements, i.e., for each argument $b \in Arg$ attacking an argument $a \in S$ there is an argument $c \in S$ attacking $b$; \emph{complete} iff it is admissible and contains all the arguments it defends; \emph{grounded} iff it is minimal (w.r.t. set inclusion) among the complete extensions; \emph{preferred} iff it is maximal among the admissible extensions; \emph{stable} iff it is conflict-free and attacks all the arguments it does not contain.

\emph{Bipolar Argumentation Frameworks}~\cite[BAF;][]{baf} extend AAFs with the notion of support between arguments. This is modelled as an additional binary relation on $Arg$, $sup$, which is disjoint from $att$. For two arguments $a_1,a_n \in Arg$, $a$ is a \emph{supported defeat} for $b$ iff there is a sequence of arguments $a_1,\dots,a_n \in Arg$ ($n \geq 1$) such that $(a_i,a_{i+1}) \in sup$ for $i\leq n-2$ and $(a_{n-1},a_n)\in att$; $a$ is a \emph{secondary defeat} for $b$ iff there is a sequence of arguments $a_1,\dots,a_n \in Arg$ ($n \geq 2$), such that $(a_1,a_2)\in att$ and $(a_i,a_{i+1}) \in sup$ for $i \geq 2$. An argument $a\in Arg$ is \emph{defended} by $S \subseteq Arg$ iff for each argument $b\in Arg$ that is a supported or secondary defeat for $a$, there exists $c \in S$ that is a supported or secondary defeat for $b$. The acceptability semantics of BAF are similar with AAF, but use the revised definitions of defeat and defense. 

\emph{Value-based Argumentation Frameworks}~\cite[VAF;][]{vaf} is another extension of AAFs that assigns (social, ethichal, etc.) values to arguments and uses preferences on values to resolve conflicts among arguments. Formally, a VAF (additionally to $Arg$ and $att$) includes a non-empty set of values, $V$, a function $val: Arg \mapsto V$ and a preference relation $valpref \subseteq V\times V$, which is transitive, irreflexive and assymetric. A preference relation on $Arg$, $pref$, is induced from $valpref$: for any two arguments $a,b \in Arg$, where $val(a)=u$ and $val(b)=v$, $(a,b)\in pref$ when $(v,u)\in valpref^+$ (the transitive closure of $valpref$). VAF uses the same acceptability semantics with AAF and BAF, but based on a different notion of defeat and defense: an argument $a$ is a \emph{defeat} for argument $b$ iff $a$ attacks $b$ and $(b,a) \notin pref$. An argument $a\in Arg$ is \emph{defended} by $S \subseteq Arg$ iff for each argument $b\in Arg$ that defeats $a$, there exists $c \in S$ that defeats $b$. 

\emph{Assumption-based Argumentation Frameworks}~\cite[ABA;][]{t:13} belong to a different family of argumentation frameworks where arguments are not treated as abstract entities. Instead, arguments are constructed from available knowledge, and attacks among arguments are derived from their internal structure. An ABA framework is a tuple $\langle \mathcal{L}, \mathcal{R}, \mathcal{A}, \overline{\phantom{a}} \rangle$, where $\mathcal{L}$ is a logical language, $\mathcal{R}$ a set of rules $\phi_1,\dots,\phi_n \rightarrow \phi$ over $\mathcal{L}$, $\mathcal{A} \subseteq \mathcal{L}$ a non-empty set of assumptions, and $\overline{\phantom{b}}$ a total mapping from $\mathcal{A}$ to $\mathcal{L}$, denoting e.g., that $\overline{a}$ is the contrary of $a$. We restrict our focus to \textit{flat} ABA, where no rule has an assumption in its head. An argument $A \vdash \psi$ is a deduction of $\psi \in \mathcal{L}$ from a set of assumptions $A \subseteq \mathcal{A}$ using a set of rules $R \in \mathcal{R}$. An argument $A \vdash \psi$ attacks another argument $A' \vdash \psi'$ iff $\psi = \overline{q}$ for some $q \in A'$. For the computation of acceptable arguments in a flat ABA  we use the acceptability semantics of AAFs.

\subsection{Learning Answer Set Programs}\label{sec:learninganswersetprograms}
\textit{Learning from Answer Sets}~\cite[LAS;][]{LawRB19} is a logic-based machine learning approach that extends Inductive Logic Programming~\citep{Muggleton1991} with methods~\citep{Lawetal2018,law2020fastlas} capable of learning from examples of interpretable knowledge represented as \textit{answer set programs}. Typically, an answer set program includes four types of rules: normal rules, choice rules, hard and weak constraints~\citep{gelfond_kahl_2014}. In this paper we consider answer set programs composed of normal rules and hard constraints. Given any (first-order logic) atoms $\asp{h}, \asp{b_1},\ldots, \asp{b_n}, \asp{c_{1}},\ldots,\asp{c_{m}}$, a \emph{normal rule} is of the form $\asp{h \codeif b_1,\ldots, b_n, \naf c_{1},\ldots,\naf c_{m}}$, where $\asp{h}$  is the \emph{head}, $\asp{b_1},\ldots, \asp{b_n}, \asp{c_{1}},\ldots,\asp{c_{m}}$ (collectively) is the \emph{body} of the rule and ``$\naf$'' represents negation as failure. Rules of the form $\asp{\codeif b_1,\ldots, b_n, \naf c_{1},\ldots,\naf c_{m}}$ are called \emph{constraints}. The Herbrand Base of an answer set program $P$, denoted $HB_P$, is the set of ground (variable free) atoms that can be formed from predicates and constants in $P$. The subsets of $HB_{P}$ are called the (Herbrand) \emph{interpretations} of $P$. Given a program $P$ and an interpretation $I$, the \emph{reduct} of $P$, denoted $P^{I}$, is constructed from the grounding of $P$ in three steps: firstly, remove rules the bodies of which contain the negation of an atom in I; secondly, remove all negative literals from the remaining rules; and finally, replace the head of any constraint with $\bot$ (where $\bot\not\in HB_{P}$). Any $I\subseteq HB_P$ is an \emph{answer set} of $P$ iff it is the minimal model of $P^{I}$. We denote with $AS(P)$ the set of answer sets of a program $P$. 

In addition to the above rules, the answer set solver \emph{clingo}~\citep{Gebser} also supports \emph{heuristic statements}. In this paper, we consider heuristic statements of the form $\asp{\hash heuristic\ a.\ [1@1, false]}$, where $\asp{a}$ is an atom. When combined with suitable flags in clingo, rather than computing the full set of answer sets, clingo will return the set of answer sets which, when projected over the ground instances of the atoms in the heuristic statements, are subset-minimal. Given any program $P$ containing heuristic statements, we write $AS^*(P)$ to refer to these subset-minimal answer sets.

The ILASP system~\citep{l:18} for solving LAS tasks has recently been updated with support for learning heuristic statements (using the \texttt{--learn-heuristics} flag). We now present a modified version of the notion of a \emph{context-dependent learning from answer sets}~\citep{LawRB16} task, formalising the heuristic learning task. 
A \textit{partial interpretation}, $e_{pi}$, is a pair of sets of ground atoms $\left \langle e^{inc}_{pi}, e^{exc}_{pi} \right \rangle$, called {\em inclusions} and {\em exclusions}, respectively. An interpretation $I$ \emph{extends} $e_{pi}$ iff $e_{pi}^{inc} \subseteq I$ and $e_{pi}^{exc} \cap I = \emptyset$. Examples come in the form of \emph{context dependent partial interpretations} (CDPI). A CDPI example $e$ is a pair $\langle e_{pi}, e_{ctx}\rangle$, where
$e_{pi}$ is a partial interpretation and $e_{ctx}$ is an answer set program called the context of $e$. An answer set program $P$ is said to \emph{accept} $e$ if there is at least one answer set $A$ in $AS^*(P \cup e_{ctx})$ that extends $e_{pi}$. A LAS framework uses \emph{mode declarations} as a form of language bias to specify the space of possible answer set programs that can be learned, referred to as \emph{hypothesis space}. Informally, a \emph{mode bias} is a pair of sets of mode declarations $\langle M_h,M_b\rangle$, where $M_h$ (resp. $M_b$) are called the head (resp. body) mode declarations. Together they characterise the predicates and type of arguments that could appear as head and body conditions of the rules to learn. For a full definition of mode bias, the reader is referred to~\citep{l:18}. 

\begin{definition}\label{def:contextdependentlearning}
A \emph{context-dependent learning from answer sets} ($ILP^{context}_{LAS}$) task is a tuple $T = \langle B, S_M, \langle E^+, E^-\rangle\rangle$, where $B$ is an answer set program, $S_M$ is a hypothesis space, $E^+$ and $E^-$ are finite sets of CDPIs called, respectively, positive and negative examples. A hypothesis $H \subseteq S_M$ is a solution of $T$ (written $H\in ILP^{context}_{LAS}(T)$) if and only if:
    (i)~${H\subseteq S_{M}}$;
    (ii)~${\forall e^+ \in E^+}$, $B\cup H$ accepts $e^{+}$; and
    (iii)~${\forall e^- \in E^-}$, $B\cup H$ does not accept $e^-$.
Such a solution is called \emph{optimal} if no smaller solution exists.
\end{definition}

\section{Unifying Framework For Learning Argumentation Semantics}\label{sec:unifying}
Our focus is on developing a unifying approach for learning definitions of semantics of argumentation frameworks from examples of arguments that are \textit{in} or \textit{out} of the extensions of given frameworks. We focus on the stable, complete, admissible grounded and preferred semantics of AAFs, BAFs and VAFs as defined in the previous section. To this aim, our framework uses the following predicates: $\asp{arg}$, $\asp{att}$, $\asp{support}$, $\asp{val}$, $\asp{valpref}$, $\asp{in}$, $\asp{out}$, $\asp{pref}$, $\asp{defeated}$, $\asp{not\_defended}$ and $\asp{supported}$, where atom $\asp{arg(X)}$ denotes that $\asp{X}$ is an argument; $\asp{att(X,Y)}$ denotes an attack from $\asp{X}$ to $\asp{Y}$; $\asp{support(X,Y)}$ denotes that $\asp{X}$ supports $\asp{Y}$; $\asp{val(X,V)}$ denotes that the value $\asp{V}$ is assigned to argument $\asp{X}$; $\asp{valpref(U,V)}$ denotes a preference of value $\asp{U}$ over value value $\asp{V}$; $\asp{in(X)}$ denotes that that argument $\asp{X}$ is in an extension, while $\asp{out(X)}$ denotes that $\asp{X}$ is out of an extension for the semantics to be learned; $\asp{pref(X,Y)}$ denotes that argument $\asp{X}$ is preferred over argument $\asp{Y}$; $\asp{defeated(X)}$, $\asp{not\_defended(X)}$ and $\asp{supported(X)}$ denote defeated, undefended and supported arguments, respectively.

Recall from Definition~\ref{def:contextdependentlearning} that in our setting a learning task is a tuple $\langle B, S_M, \langle E^+, E^-\rangle\rangle$. We now define these components in the context of our unifying framework for learning argumentation semantics. We refer to this unifying framework as 
$LAS_{arg}$.

The set of positive and negative context-dependent examples, $E^+$ and $E^-$, of $LAS_{arg}$ include examples that have as inclusion and exclusion sets, \textit{in} and \textit{out} facts; and as context, a set of facts encoding an argumentation framework. Consider, for example, the case of an AAF. A positive and a negative context-dependent example are, for instance, $e^+ = \#pos(\{out(a),$ $out(b)\},$ $\{in(a),$ $in(b)\},$ $\{arg(a).$ $arg(b).$ $att(a,b).$ $att(b,a).\})$ as a positive example, and
$e^- = \#neg(\{in(a),$ $in(b)\},$ $\{\},$ $\{arg(a).$ $arg(b).$ $att(a,b).$ $att(b,a).\})$ as a negative example. In these $\# pos$ and $\# neg$ statements, the first set specifies the inclusions, the second set specifies the exclusions, and the third set defines the context. In both examples, the context contains two arguments, $a$ and $b$, that attack each other. In the positive example $e^+$, both arguments are labelled as being $\asp{out}$, but not $\asp{in}$. In the negative example $e^-$, $a$ and $b$ are labelled as $\asp{in}$ in the inclusion set, indicating that there should be no extension of the argumentation framework that includes both arguments. In the case of BAF and VAF, examples may also contain \asp{support}, \asp{val} or \asp{valpref} facts in the context. 

The background knowledge $B$ of $LAS_{arg}$ includes additional definitions for special constructs of the argumentation frameworks under consideration. Specifically, it contains definitions for $\asp{defeated}$, $\asp{not\_defended}$ and $\asp{supported}$ arguments, and a $\asp{pref}$ relation over arguments. The most general form of background knowledge is shown in Listing~\ref{lst:backgroundknowledge}.

\begin{table}[h]
\setlength{\tabcolsep}{3pt}
\refstepcounter{table}\label{lst:backgroundknowledge}
\begin{tabular}{rl}
\hline
\multicolumn{2}{l}{Listing 1: Background knowledge $B$} \\
\hline
1 & $\asp{support(X,Z) \codeif support(X,Y),support(Y,Z).}$\\
2 & $\asp{supported(X)\codeif support(Y,X),in(Y).}$\\
3 & $\asp{valpref(X,Y)\codeif valpref(X,Z),valpref(Z,Y).}$\\
4 & $\asp{pref(X,Y)\codeif valpref(U,V),val(X,U),val(Y,V).}$\\
5 & $\asp{pref(X,Y)\codeif pref(X,Z),pref(Z,Y).}$\\
6 & $\asp{defeat(X,Y)\codeif att(Z,Y),support(X,Z).}$\\ 
7 & $\asp{defeat(X,Y)\codeif att(X,Z),support(Z,Y).}$\\
8 & $\asp{defeat(X,Y)\codeif att(X,Y),not\; pref(Y,X).}$\\
9 & $\asp{defeated(X)\codeif in(Y),defeat(Y,X).}$\\
10 & $\asp{not\_defended(X)\codeif defeat(Y,X), not\; defeated(Y).}$\\
\hline
\end{tabular}
\end{table} 

The hypothesis space $S_M$ of $LAS_{arg}$ is defined as to capture the search space of the rules that might encode possible argumentation semantics. We define the hypothesis space $S_M$ through \textit{mode declarations}, as shown in Listing~\ref{lst:hypothesisspace}.
\begin{table}
\setlength{\tabcolsep}{3pt}
\refstepcounter{table}\label{lst:hypothesisspace}
\begin{tabular}{rl}
\hline
\multicolumn{2}{l}{Listing 2: Hypothesis space $S_M$} \\
\hline
1 & $\asp{\hash modeh(in(var(arg))).}$\\
2 & $\asp{\hash modeh(out(var(arg))).}$\\
3 & $\asp{\hash modeb(in(var(arg))).}$\\
4 & $\asp{\hash modeb(out(var(arg))).}$\\
5 & $\asp{\hash modeb(arg(var(arg)), (positive)).}$\\ 
6 & $\asp{\hash modeb(att(var(arg), var(arg))).}$\\
7 & $\asp{\hash modeb(defeated(var(arg))).}$\\ 
8 & $\asp{\hash modeb(not\_defended(var(arg))).}$\\
9 & $\asp{\hash modeb(supported(var(arg)).}$\\
\hline
\end{tabular}
\end{table}

Since we look for definitions for the predicates $\asp{in}$ and $\asp{out}$, these predicates appear as arguments of a mode head declaration (Lines 1--2 in Listing~\ref{lst:hypothesisspace}). All the predicates of our unifying framework can appear in the body of the learned rules, and are therefore included as arguments of mode body declarations (Lines 3--8 in Listing~\ref{lst:hypothesisspace}). 

Given a $LAS_{arg}$ task $T$, a learned argumentation semantics is an optimal solution of $T$ according to Definition~\ref{def:contextdependentlearning}. Such solution can be computed by the ILASP system\footnote{We run ILASP with the command \texttt{ILASP --version=4 --learn-heuristics semantics.las}, where \texttt{semantics.las} is the filename with the positive and negative examples for a given semantics.}.

\subsection{Learning the Semantics of AAFs}
In this section, we illustrate in detail the case of learning the semantics of AAF. Positive and negative examples have inclusion sets that contain \emph{in} and \emph{out} labelings of the arguments, specified in the context, which are consistent with the given AAF semantics. As mentioned in the introduction, we manually engineered positive and negative examples to show that our method can learn specific known semantics when the training set is appropriate. For instance, the positive examples we crafted for the admissible semantics are: $\#pos(\{in(a)$, $in(b)$, $out(c)\}$, $\{out(a)$, $out(b)$, $in(c)\}$, $\{arg(a).$ $arg(b).$ $arg(c).\})$, $\#pos(\{in(a)$, $out(b)\}$, $\{out(a)$, $in(b)\}$, $\{arg(a).$ $arg(b).\})$, $\#pos(\{in(a)$, $out(b)\}$, $\{out(a)$, $in(b)\}$, $\{arg(a).$ $arg(b).$ $att(a,b).$ $att(b, a).\})$, $\#pos(\{in(a)$, $out(b)$, $in(c)\}$, $\{out(a)$, $in(b)$, $out(c)\}$, $\{arg(a).$ $arg(b).$ $arg(c).$ $att(a, b).$ $att(b, c).\})$, and $\#pos(\{in(a)$, $out(b)$, $in(c)$, $out(d)\}$, $\{out(a)$, $in(b)$, $out(c)$, $in(d)\}$, $\{arg(a).$ $arg(b).$ $arg(c).$ $arg(d).$ $att(a, b).$ $att(b,c).\})$ which instantiate argumentation frameworks and some of their admissible extensions. On the other hand, negative examples are $\#neg(\{in(a)$, $in(b)$, $out(c)$, $out(d)\}$, $\{\}$, $\{arg(a).$ $arg(b).$ $arg(c).$ $arg(d).$ $att(a,b).$ $att(b,a).$ $att(b,c).$ $att(a,c).$ $att(c,d).\})$, and $\#neg(\{out(a)$, $out(b)$, $in(c)$, $out(d)\}$, $\{\}$, $\{arg(a).$ $arg(b).$ $arg(c).$ $arg(d).$ $att(a,b).$ $att(b,a).$ $att(b,c).$ $att(a,c).$ $att(c,d).\})$, that encode extensions that are not admissible in the corresponding argumentation framework.

The background knowledge of the unified LAS task encoding can be simplified for AAF. These simplifications make it possible to obtain more compact encodings, thus improving their overall performance in both learning and inference. Consider the background knowledge $B$ given in Listing~\ref{lst:backgroundknowledge}. Since AAFs do not use the predicates $\asp{valpref}$, $\asp{val}$ and $\asp{support}$, examples of learning AAFs semantics will not include any fact that uses these predicates. Rules in lines 1--5 do not produce any instances of the predicates $\asp{supported}$ and $\asp{pref}$. Similarly, rules in lines 6 and 7 do not produce any instances of the predicate $\asp{defeat}$. Only the rule in line 8 may produce instances of $\asp{defeat}$. As there are no instances of the predicate $\asp{pref}$, the rule in line 8 is equivalent to $\asp{defeat(Y,X)\codeif att(Y,X)}$. From this simplified rule, if we have $\asp{att(Y,X)}$, then we also have $\asp{defeat(Y,X)}$, and this is the only definition of the predicate $\asp{defeat}$ that can produce it. This means that we can rewrite the rule in line 9 as $\asp{defeated(X)\codeif in(Y), att(Y,X)}$ and the rule in line 10 as
$\asp{not\_defended(X)\codeif att(Y,X),not\; defeated(Y)}$. The definition of $\asp{defeat}$ (rule in line 8) can be dropped since this predicate does not occur in the body of any rule. Thus, in the case of AAF, the simplified background knowledge will include only two  rules, instead of the $10$ rules included in Listing~\ref{lst:backgroundknowledge}. We refer to the simplified background knowledge for AAF as $B_{AAF}$, and show it in Listing~\ref{lst:simplifiedbackground}.

\begin{table}[h]
\setlength{\tabcolsep}{3pt}
\refstepcounter{table}\label{lst:simplifiedbackground}
\begin{tabular}{rl}
\hline
\multicolumn{2}{l}{Listing 3: Simplified background knowledge $B_{AAF}$} \\
\hline
1 & $\asp{defeated(X)\codeif in(Y), att(Y,X).}$\\
2 & $\asp{not\_defended(X)\codeif att(Y,X), not\; defeated(Y).}$\\
\hline
\end{tabular}
\end{table}

From sets of positive and negative examples that are consistent with AAF semantics and using the simplified background knowledge $B_{AAF}$, our $LAS_{arg}$ task accepts as optimal solutions the answer set programs given in listings~\ref{lst:stablesemantics}--\ref{lst:preferredemantics} for the stable, the complete, the admissible, the grounded and the preferred semantics, respectively. These learned programs, together with the simplified background knowledge $B_{AAF}$, provide the full learned AAF semantics. Note that the same optimal hypotheses would be learned if the initial full background knowledge $B$ was used instead.

\begin{table}[h]
\setlength{\tabcolsep}{3pt}
\refstepcounter{table}\label{lst:stablesemantics}
\begin{tabular}{rl}
\hline
\multicolumn{2}{l}{Listing 4: Stable semantics} \\
\hline
1 & $\asp{out(X)\codeif defeated(X).}$\\
2 & $\asp{in(X)\codeif arg(X), not\; out(X).}$\\
\hline
\end{tabular}
\end{table}

\begin{table}[h]
\setlength{\tabcolsep}{3pt}
\refstepcounter{table}\label{lst:completesemantics}
\begin{tabular}{rl}
\hline
\multicolumn{2}{l}{Listing 5: Complete semantics} \\
\hline
1 & $\asp{out(X)\codeif not\_defended(X).}$\\
2 & $\asp{in(X)\codeif arg(X), not\; out(X), not\; defeated(X).}$\\
\hline
\end{tabular}
\end{table}

\begin{table}[h]
\setlength{\tabcolsep}{3pt}
\refstepcounter{table}\label{lst:admissiblesemantics}
\begin{tabular}{rl}
\hline
\multicolumn{2}{l}{Listing 6: Admissible semantics} \\
\hline
1 & $\asp{out(X)\codeif defeated(X).}$\\
2 & $\asp{out(X)\codeif arg(X), not\; in(X).}$\\
3 & $\asp{in(X)\codeif arg(X), not\; out(X), not\; not\_defended(X).}$\\
\hline
\end{tabular}
\end{table}

\begin{table}[h]
\setlength{\tabcolsep}{3pt}
\refstepcounter{table}\label{lst:groundedsemantics}
\begin{tabular}{rl}
\hline
\multicolumn{2}{l}{Listing 7: Grounded semantics} \\
\hline
1 & $\asp{in(X)\codeif arg(X), not\; not\_defended(X).}$\\
2 & $\asp{out(X)\codeif not\_defended(X).}$\\
3 & $\asp{\hash heuristic\; in(X).\; [1@1, false]}$\\
\hline
\end{tabular}
\end{table}

\begin{table}[h]
\setlength{\tabcolsep}{3pt}
\refstepcounter{table}\label{lst:preferredemantics}
\begin{tabular}{rl}
\hline
\multicolumn{2}{l}{Listing 8: Preferred semantics} \\
\hline
1 & $\asp{in(X)\codeif arg(X), not\; defeated(X)}$, \\
& \hspace{11.5mm} $\asp{not\; not\_defended(X).}$\\
2 & $\asp{out(X)\codeif not\_defended(X).}$\\
3 & $\asp{\hash heuristic\; out(X).\; [1@1, false]}$\\
\hline
\end{tabular}
\end{table}

We have used a similar approach to learn the semantics of BAF and VAF. The background knowledge of our unifying framework can be similarly simplified by obtaining the $B_{BAF}$ and $B_{VAF}$ programs given in  Listings~\ref{lst:simplifiedbaf} and~\ref{lst:simplifiedvaf}, respectively.

Interestingly, the learned solutions for BAF and VAF are the same as those of AAF. By adding these learned solutions to the simplified background knowledge, $B_{BAF}$ and $B_{VAF}$, we get the corresponding full learned semantics for BAF and VAF, respectively.  

\subsubsection{Equivalence with ASPARTIX}
The following theorem proves the equivalence between our learned admissible, complete, and stable AAF semantics and the corresponding ASPARTIX encodings, showing the soundness of our unifying learning framework with respect to AAF.

\begin{theorem}
    Let $F$ be an ASP representation of an AAF. Let $T_{\sigma}=\langle B_{AAF},S_M,\langle E^+,E^-\rangle \rangle$ be a learning task for the $\sigma$-semantics of AAF, where $\sigma$ stands for admissible, complete or stable. Let $H_{\sigma}$ be a solution to $T_{\sigma}$, and  $P_{\sigma}=B_{AAF} \cup H_{\sigma}$ be the full learned $\sigma$-semantics. Let $S_{\sigma}$ be the ASP encoding of the $\sigma$-semantics in ASPARTIX. Then, $AS(P_{\sigma}\cup F)=AS(S_{\sigma}\cup F).$
\end{theorem}

\begin{proof}
We prove the case where $\sigma$ is the admissible semantics\footnote{The other cases can be proved in a similar fashion, and are included in the fuller report available at \href{https://github.com/dasaro/ArgLAS}{\texttt{https://github.com/dasaro/ArgLAS}}.}. This amounts to showing that, given the encoding of an AAF as a set of \asp{arg} and \asp{att} facts $F$, $A$ is an answer set of $S_{adm}\cup F$ if and only if $A$ is an answer set of $P_{adm}\cup F.$ The program $S_{adm}$, i.e., the ASPARTIX encoding of the admissible semantics, is
\begin{mdframed}[backgroundcolor=light-gray, roundcorner=10pt,leftmargin=0.1, rightmargin=0.1, innerleftmargin=5, innertopmargin=1,innerbottommargin=1, outerlinewidth=1, linecolor=light-gray]
\begin{lstlisting}[mathescape]
in(X) :- not out(X), arg(X).
out(X) :- not in(X), arg(X).
defeated(X) :- in(Y), att(Y,X).
not_defended(X) :- att(Y,X),
    not defeated(Y).
$\bot$ :- in(X), in(Y), att(X,Y).
$\bot$ :- in(X), not_defended(X).
\end{lstlisting}
\end{mdframed}
and we remind the reader that the program $P_{adm}$, i.e., our learned encoding is:
\begin{mdframed}[backgroundcolor=light-gray, roundcorner=10pt,leftmargin=0.1, rightmargin=0.1, innerleftmargin=5, innertopmargin=1,innerbottommargin=1, outerlinewidth=1, linecolor=light-gray]
\begin{lstlisting}[mathescape]
out(X) :- defeated(X).
out(X) :- arg(X), not in(X).
in(X) :- arg(X), not out(X),
    not not_defended(X).
defeated(X) :- in(Y), att(Y,X).
not_defended(X) :- att(Y,X),
    not defeated(Y).
\end{lstlisting}
\end{mdframed}

    ($\Rightarrow$): We first assume that $A$ is an answer set of $S_{adm}\cup F$, and show that $A$ is also an answer set of $P_{adm}\cup F.$ Again for simplicity, as there are no function symbols, we can assume that $ground(P_{adm})$ returns all the ground instances of $P_{adm}$. The reduct $(ground(P_{adm}\cup F))^A$ must then contain:

\begin{mdframed}[backgroundcolor=light-gray, roundcorner=10pt,leftmargin=0.1, rightmargin=0.1, innerleftmargin=5, innertopmargin=1,innerbottommargin=1, outerlinewidth=1, linecolor=light-gray]
\begin{lstlisting}[mathescape]
(1) F.
(2) out(a) :- defeated(a), for each constant a.
(3) out(a) :- arg(a), for each constant a such that in(a)$\not\in A$.
(4) in(a) :- arg(a), for each constant a such that out(a)$\not\in A$ and not_defended(a)$\not\in A$.
(5) defeated(a) :- att(b,a), in(b), for each pair of constants a and b.
(6) not_defended(a) :- att(b,a), for each pair of constants a and b, such that defeated(b)$\not\in A$.
\end{lstlisting}
\end{mdframed} 

We see that \asp{(2)} can be rewritten as:
 
\begin{mdframed}[backgroundcolor=light-gray, roundcorner=10pt,leftmargin=0.1, rightmargin=0.1, innerleftmargin=5, innertopmargin=1,innerbottommargin=1, outerlinewidth=1, linecolor=light-gray]
\begin{lstlisting}[mathescape]
(2) out(a), for each constant a such that defeated(a)$\in A$.
\end{lstlisting}
\end{mdframed} 

Since $A$ is an answer set of $S_{adm}\cup F$, the first two rules in $S_{adm}$ imply that for each constant \asp{a}, \asp{out(a)}$\in A$ iff \asp{in(a)}$\not\in A$ and \asp{in(a)}$\in A$ iff \asp{out(a)}$\not\in A$. This allows us to rewrite rule \asp{(3)} as:
 
\begin{mdframed}[backgroundcolor=light-gray, roundcorner=10pt,leftmargin=0.1, rightmargin=0.1, innerleftmargin=5, innertopmargin=1,innerbottommargin=1, outerlinewidth=1, linecolor=light-gray]
\begin{lstlisting}[mathescape]
(3) out(a), for each constant a such that out(a)$\in A$. 
\end{lstlisting}
\end{mdframed} 

We can now combine \asp{(2)} and \asp{(3)} to get:
 
\begin{mdframed}[backgroundcolor=light-gray, roundcorner=10pt,leftmargin=0.1, rightmargin=0.1, innerleftmargin=5, innertopmargin=1,innerbottommargin=1, outerlinewidth=1, linecolor=light-gray]
\begin{lstlisting}[mathescape]
(2-3) out(a), for each constant a such 
      that out(a)$\in A$ or defeated(a)$\in A$. 
\end{lstlisting}
\end{mdframed} 
 
To simplify this rule further, we will make use of the following Lemma.
\begin{lemma}\label{lem:4.1.3}
Given $A$ an answer set of $S_{adm}\cup F$ and \asp{a} an argument constant, \asp{out(a)}$\in A$ iff \asp{out(a)}$\in A$ or \asp{defeated(a)}$\in A$.
\end{lemma}
\begin{innerproof}[Proof of Lemma]
    The proof from left to right follows from the observation that if we have \asp{out(a)}$\in A$, then \asp{out(a)}$\in A$ or \asp{defeated(a)}$\in A.$
    
    For the other direction, we have to separately show that if \asp{out(a)}$\in A$, then \asp{out(a)}$\in A$ and that if \asp{defeated(a)}$\in A$, then \asp{out(a)}$\in A$. Clearly, if \asp{out(a)}$\in A$, then \asp{out(a)}$\in A$. Let's assume \asp{defeated(a)}$\in A$. The first constraint in $S_{adm}$ is equivalent to \asp{:- in(Y), defeated(Y)}, where the definition of the \asp{defeated} predicate is used. Then having \asp{defeated(a)}$\in A$ means that \asp{in(a)} is not in $A$ and by the second rule in $S_{adm}$, \asp{out(a)}$\in A$. Hence, \asp{out(a)}$\in A$ iff \asp{out(a)}$\in A$ or \asp{defeated(a)}$\in A$.\qed 
\end{innerproof}

Applying Lemma \ref{lem:4.1.3}, the rule \asp{(2-3)} can be rewritten as:
\begin{mdframed}[backgroundcolor=light-gray, roundcorner=10pt,leftmargin=0.1, rightmargin=0.1, innerleftmargin=5, innertopmargin=1,innerbottommargin=1, outerlinewidth=1, linecolor=light-gray]
\begin{lstlisting}[mathescape]
(2-3) out(a), for each constant a such that out(a)$\in A$.
\end{lstlisting}
\end{mdframed} 

We can also transform \asp{(4)} to:
 
\begin{mdframed}[backgroundcolor=light-gray, roundcorner=10pt,leftmargin=0.1, rightmargin=0.1, innerleftmargin=5, innertopmargin=1,innerbottommargin=1, outerlinewidth=1, linecolor=light-gray]
\begin{lstlisting}[mathescape]
(4) in(a), for each constant a such 
    that in(a)$\in A$ and not_defended(a)$\not\in A$.
\end{lstlisting}
\end{mdframed} 
 
We again proceed with proving a Lemma to transform this rule further.
\begin{lemma}\label{lem:4.1.4}
Given that $A$ is an answer set of $S_{adm}\cup F$ and \asp{a} is an argument constant, \asp{in(a)}$\in A$ iff \asp{in(a)}$\in A$ and \asp{not\_defended(a)}$\not\in A$.
\end{lemma}
\begin{innerproof}[Proof of Lemma]
    The proof from right to left follows from the observation that if we have \asp{in(a)}$\in A$ and \asp{not\_defended(a)}$\not\in A$, then, in particular, \asp{in(a)}$\in A$.
    
    For the other direction, we are left to show that if \asp{in(a)}$\in A$, then we also have \asp{not\_defended(a)}$\not\in A$. For this we can use the last constraint in $S_{adm}$: if \asp{in(a)}$\in A$, then \asp{not\_defended(a)}$\not\in A$. It follows that if \asp{in(a)}$\in A$, then \asp{in(a)}$\in A$ and \asp{not\_defended(a)}$\not\in A$.\qed 
\end{innerproof}

Applying Lemma \ref{lem:4.1.4}, rule \asp{(4)} is equivalent to:
 
\begin{mdframed}[backgroundcolor=light-gray, roundcorner=10pt,leftmargin=0.1, rightmargin=0.1, innerleftmargin=5, innertopmargin=1,innerbottommargin=1, outerlinewidth=1, linecolor=light-gray]
\begin{lstlisting}[mathescape]
(4) in(a), for each constant a such that in(a)$\in A$.
\end{lstlisting}
\end{mdframed} 
 
Since everything in the body has been proven to be the same in the definitions of \asp{defeated} in both $S_{adm}$ and $P_{adm}$, rule \asp{(5)} can be simplified to:
 
\begin{mdframed}[backgroundcolor=light-gray, roundcorner=10pt,leftmargin=0.1, rightmargin=0.1, innerleftmargin=5, innertopmargin=1,innerbottommargin=1, outerlinewidth=1, linecolor=light-gray]
\begin{lstlisting}[mathescape]
(5) defeated(a), for each constant a such that defeated(a)$\in A$.
\end{lstlisting}
\end{mdframed}

This also means that everything in the body of rule \asp{(6)} is the same in both cases, so \asp{(6)} becomes
 
\begin{mdframed}[backgroundcolor=light-gray, roundcorner=10pt,leftmargin=0.1, rightmargin=0.1, innerleftmargin=5, innertopmargin=1,innerbottommargin=1, outerlinewidth=1, linecolor=light-gray]
\begin{lstlisting}[mathescape]
(6) not_defended(a), for each constant a 
    such that not_defended(a)$\in A.$
\end{lstlisting}
\end{mdframed}

We have now simplified each rule so that we can conclude that all ground instances of the predicates \asp{arg}, \asp{att}, \asp{in}, \asp{out}, \asp{defeated} and \asp{not\_defended} are the same in the reduct and in $A$. Hence, the minimal model of the reduct must be equal to $A$, meaning that $A$ is an answer set of $P_{adm}\cup F.$\\

($\Leftarrow$): Now we are left with showing that if $A$ is an answer set of $P_{adm} \cup F$, it must also be an answer set of $S_{adm} \cup F$. For the purpose of the proof, let's assume that $A$ is an answer set of $P_{adm} \cup F$. The reduct of $ground((S_{adm} \cup F)^A$ contains the following:
 
\begin{mdframed}[backgroundcolor=light-gray, roundcorner=10pt,leftmargin=0.1, rightmargin=0.1, innerleftmargin=5, innertopmargin=1,innerbottommargin=1, outerlinewidth=1, linecolor=light-gray]
\begin{lstlisting}[mathescape]
(1) F.
(2) in(a), for each constant a such that out(a)$\not\in A$.
(3) out(a), for each constant a such that in(a)$\not\in A$.
(4) defeated(a) :- in(b), att(b,a), for each pair 
    of constants a and b.
(5) not_defended(a) :- att(b,a), for each pair of
    constants a and b, such that defeated(b)$\not\in A$.
(6) $\bot$ :- in(a), in(b), att(a,b), for each 
    pair of constants a and b.
(7) $\bot$ :- in(a), not_defended(a), for each constant a.
\end{lstlisting}
\end{mdframed}
where for simplifying \asp{(2)} and \asp{(3)} we used the fact that each constant \asp{a} has \asp{arg(a)} in $F$.

We have to show that the constraints \asp{(6)} and \asp{(7)} are satisfied by $A$, which means that they should not produce $\bot$ ($\bot$ cannot be a part of any model). For \asp{(6)}, first note that it is equivalent to:
 
\begin{mdframed}[backgroundcolor=light-gray, roundcorner=10pt,leftmargin=0.1, rightmargin=0.1, innerleftmargin=5, innertopmargin=1,innerbottommargin=1, outerlinewidth=1, linecolor=light-gray]
\begin{lstlisting}[mathescape]
(6) $\bot$ :- in(b), defeated(b).
\end{lstlisting}
\end{mdframed}
where we make use of the definition of the \asp{defeated} predicate. For this constraint to be violated, there must be a constant \asp{b} such that \asp{in(b)}$\in A$ and \asp{defeated(b)}$\in A$. But if \asp{defeated(b)}$\in A$, then by the first rule in $P_{adm}$, \asp{out(b)} must also be in $A$, meaning that \asp{in(b)} could not be in $A$ by the third rule in $P_{adm}$. Hence, the constraint cannot be violated.

Rule \asp{(7)} cannot produce $\bot$ using the ground instances of $A$, as \asp{in(X)} is defined to only be true if \asp{not\_defended(X)} is false (by the third rule in $P_{adm}$).

To simplify rule \asp{(2)} we prove the following Lemma.
\begin{lemma}\label{lem:4.1.5}
Given that $A$ is an answer set of $P_{adm}\cup F$ and $a$ is an argument constant, \asp{out(a)}$\not\in A$ iff \asp{in(a)} $\in A$.
\end{lemma}
\begin{innerproof}[Proof of Lemma]
    For the left to right direction, notice that if \asp{in(a)}$\not\in A$ the second rule of $P_{adm}$ will produce \asp{out(a)}. This means that if \asp{out(a)}$\not\in A$, then \asp{in(a)}$\in A.$
    
    For the right to left direction, we use the third rule in $P_{adm}$: if we have \asp{in(a)}$\in A$, then \asp{out(a)} should not be in $A$.\qed
\end{innerproof}

Using Lemma \ref{lem:4.1.5}, rule \asp{(2)} is transformed to
 
\begin{mdframed}[backgroundcolor=light-gray, roundcorner=10pt,leftmargin=0.1, rightmargin=0.1, innerleftmargin=5, innertopmargin=1,innerbottommargin=1, outerlinewidth=1, linecolor=light-gray]
\begin{lstlisting}[mathescape]
(2) in(a), for each constant a such that in(a)$\in A$.
\end{lstlisting}
\end{mdframed} 

Also, by the second rule in $P_{adm}$, if \asp{in(X)}$\not\in A$, then \asp{out(X)}$\in A$, so rule \asp{(3)} becomes:
 
\begin{mdframed}[backgroundcolor=light-gray, roundcorner=10pt,leftmargin=0.1, rightmargin=0.1, innerleftmargin=5, innertopmargin=1,innerbottommargin=1, outerlinewidth=1, linecolor=light-gray]
\begin{lstlisting}[mathescape]
(3) out(a), for each constant a such that out(a)$\in A.$
\end{lstlisting}
\end{mdframed} 

Every body condition in the definition of the \asp{defeated} predicate is the same for both $S_{adm}$ and $P_{adm}$, so \asp{(4)} can be rewritten as:
 
\begin{mdframed}[backgroundcolor=light-gray, roundcorner=10pt,leftmargin=0.1, rightmargin=0.1, innerleftmargin=5, innertopmargin=1,innerbottommargin=1, outerlinewidth=1, linecolor=light-gray]
\begin{lstlisting}[mathescape]
(4) defeated(a), for defeated(a)$\in A.$
\end{lstlisting}
\end{mdframed} 

Everything in the body of rule \asp{(5)} (the definition of the \asp{not\_ defended} predicate) is now the same in both cases so \asp{(5)} can be simplified to:
 
\begin{mdframed}[backgroundcolor=light-gray, roundcorner=10pt,leftmargin=0.1, rightmargin=0.1, innerleftmargin=5, innertopmargin=1,innerbottommargin=1, outerlinewidth=1, linecolor=light-gray]
\begin{lstlisting}[mathescape]
(5) not_defended(a), for not_defended(a)$\in A.$
\end{lstlisting}
\end{mdframed} 
 
We have simplified each rule and are now able conclude that all ground instances of the predicates \asp{arg}, \asp{att}, \asp{in}, \asp{out}, \asp{defeated} and \asp{not\_defended} are the same in the reduct and in $A$. Moreover, the two constraints \asp{(6)} and \asp{(7)} never produce $\bot$, which means that the minimal model of the reduct must be equal to $A$, and so $A$ is an answer set of $S_{adm}\cup F.$

Given the proofs of ($\Rightarrow$) and ($\Leftarrow$), $A$ is an answer set of $S_{adm}\cup F$ iff $A$ is an answer set of $P_{adm}\cup F.$
\end{proof}

\begin{table}[h]
\setlength{\tabcolsep}{3pt}
\refstepcounter{table}\label{lst:simplifiedbaf}
\begin{tabular}{rl}
\hline
\multicolumn{2}{l}{Listing 9: Simplified background knowledge $B_{BAF}$} \\
\hline
1 & $\asp{support(X,Z)\codeif support(X,Y),support(Y,Z).}$\\
2 & $\asp{supported(X)\codeif support(Y,X), in(Y).}$\\
3 & $\asp{defeat(X,Y)\codeif att(Z,Y), support(X,Z).}$\\
4 & $\asp{defeat(X,Y)\codeif att(X,Z), support(Z,Y).}$\\
5 & $\asp{defeat(X,Y)\codeif att(X,Y).}$\\
6 & $\asp{defeated(X)\codeif in(Y), defeat(Y,X).}$\\ 
7 & $\asp{not\_defended(X)\codeif defeat(Y,X), not\; defeated(Y).}$\\
\hline
\end{tabular}
\end{table}

\begin{table}[h]
\setlength{\tabcolsep}{3pt}
\refstepcounter{table}\label{lst:simplifiedvaf}
\begin{tabular}{rl}
\hline
\multicolumn{2}{l}{Listing 10: Simplified background knowledge $B_{VAF}$} \\
\hline
1 & $\asp{valpref(X,Y)\codeif valpref(X,Z),valpref(Z,Y).}$\\
2 & $\asp{pref(X,Y)\codeif valpref(U,V),val(X,U),val(Y,V).}$\\
3 & $\asp{pref(X,Y)\codeif pref(X,Z),pref(Z,Y).}$\\
4 & $\asp{defeat(X,Y)\codeif att(X,Y),not\; pref(Y,X).}$\\
5 & $\asp{defeated(X)\codeif in(Y),defeat(Y,X).}$\\
6 & $\asp{not\_defended(X)\codeif defeat(Y,X), not\; defeated(Y).}$\\
\hline
\end{tabular}
\end{table} 

\subsection{Applying $LAS_{arg}$ to Flat ABA Frameworks}
In this section, we show how our unified framework $LAS_{arg}$ can also be applied to ABA. Recall that ABA involves structured arguments. The definition of ABA includes rules, assumptions and a contrary relation. To use our unified $LAS_{arg}$ task, we have developed a 2-step algorithm for translating a flat ABA framework into an AAF. The first step constructs arguments from assumptions and rules. The second step uses the contrary relation to compute the attack relations. The unified $LAS_{arg}$ task can then be applied to the resulting AAF.

\subsubsection{Step 1: Constructing Arguments.} We write an answer set program that encodes rules and assumptions of a given ABA framework and produces the arguments of the framework as answer sets. Each argument has one root and one or more assumptions. To represent assumptions, we use the predicate $\asp{as(X)}$. To represent rules, we use the predicate $\asp{holds(X)}$. For example, we encode the rule $r \gets s,t$ as $\asp{holds(r):-holds(s),holds(t)}$. We also use $\asp{root(X)}$ to express that $X$ is the root of the argument, and $\asp{assume(X)}$ to express that $X$ is an assumption in the argument. 

Consider, for example, the ABA framework defined by $\mathcal{L}=\{ p,q,r,s,t \}$, $\mathcal{R} = \{ r \leftarrow s,t,\,\,s \leftarrow p,\,\,t \leftarrow q \}$, $\mathcal{A}=\{p,q\}$, $\bar{p}=t$ and $\bar{q}=r$. This translates to lines 1--5 in Listing~\ref{lst:constructarguments}. Lines 6--10 are auxiliary definitions required for the learning task.
\begin{table}[h]
\setlength{\tabcolsep}{3pt}
\refstepcounter{table}\label{lst:constructarguments}
\begin{tabular}{rl}
\hline
\multicolumn{2}{l}{Listing 11: Step 1 - Construct Arguments} \\
\hline
1 & $\asp{as(p).}$\\
2 & $\asp{as(q).}$\\
3 & $\asp{holds(r)\codeif holds(s), holds(t).}$\\
4 & $\asp{holds(s)\codeif holds(p).}$\\
5 & $\asp{holds(t)\codeif holds(q).}$\\
6 & $\asp{0\{assume(X)\}1\codeif as(X).}$\\
7 & $\asp{holds(X)\codeif assume(X).}$\\
8 & $\asp{1\{root(X): holds(X)\}1.}$\\
9 & $\asp{\hash heuristic\; assume(X).\;  [1, false]}$\\
10 & $\asp{\hash heuristic\; root(X).\;  [1, true]}$\\
11 & $\asp{\hash show\; root/1.}$\\
12 & $\asp{\hash show\; assume/1.}$\\
\hline
\end{tabular}
\end{table}
For each assumption, we decide whether to assume it (Line 6 in Listing 11) and if we assume it, then it holds (Line 7 in Listing 11). Moreover, we have exactly one root, but for something to be a root, it has to hold (Line 8 in Listing 11). Heuristics statements in Lines 9 and 10 are used to favor answer sets that set the atom $\asp{assume(X)}$ to false and $\asp{root(X)}$ to true. During the search for answer sets, the solver sets one of the $\asp{assume}$ or $\asp{root}$ atoms to the desired truth value (false for $\asp{assume}$ and true for $\asp{root}$). By running the solver with the enumeration mode flag (\texttt{--enum=domrec}), whenever an answer set is found, the solver adds constraints stating that any further answer sets must be ``better'' in at least one way -- either by making one of the $\asp{assume}$ atoms that is true in the previous answer set false, or by making one of the $\asp{root}$ atoms that is false in the previous answer set true. As there is exactly one $\asp{root}$ atom per answer set, all answer sets that are minimal over the $\asp{assume}$ atom for each root are computed. This is exactly what we are looking for, since to construct the arguments, we want to find one proof for each root (starting from assumptions), that involves a minimal number of assumptions.

Each answer set of the program in Listing~\ref{lst:constructarguments} describes an argument with the predicates $\asp{root(X)}$ and $\asp{assume(X)}$, which give the root and the assumptions in the argument. When we solve this program with clingo\footnote{We run clingo with support for heuristics with the command \texttt{clingo -n 1 constr\_args.lp --heuristic=domain --enum=domrec}, where \texttt{constr\_args.lp} is the program in Listing~\ref{lst:constructarguments}.}, we obtain the following answer sets: $\asp{\{assume(p),root(p)\}, \{assume(q),}$ $\asp{root(q)\}, \{assume(p),root(s)\}, \{assume(q),root(t)\}}$, \\
$\asp{\{assume(p),assume(q),root(r)\}}$.

\subsection{Step 2: Finding the Attack Relations}
After completing Step 1, we assign an index to each argument, which makes them easier to encode. We follow the output ordering of the answer sets in clingo. In the example above, this process results in the following assignments: $1$ to $\asp{\{assume(p),root(p)\}}$, $2$ to $\asp{\{assume(q),root(q)\}}$, $3$ to $\asp{\{assume(p),root(s)\}}$, $4$ to $\asp{\{assume(q),root(t)\}}$ and $5$ to $\asp{\{assume(p),assume(q),root(r)\}}$. Let $\asp{root(N, X)}$ denote that $X$ is a root in the argument with index $N$, and let $\asp{as(N, X)}$ denote that $X$ is an assumption in the same argument. Let $\asp{contr(P,Q)}$ express that $P$ is the contrary of $Q$ in the ABA framework. Thus, the framework in the example above translates to lines 1--11 of Listing~\ref{lst:generateattacks}.
\begin{table}[h]
\setlength{\tabcolsep}{3pt}
\refstepcounter{table}\label{lst:generateattacks}
\begin{tabular}{rl}
\hline
\multicolumn{2}{l}{Listing 12: Step 2 - Generate Attacks} \\
\hline
1 & $\asp{root(3,s).}$\\
2 & $\asp{root(4,t).}$\\
3 & $\asp{root(5,r).}$\\
4 & $\asp{as(1,p).}$\\
5 & $\asp{as(2,q).}$\\
6 & $\asp{as(3,p)}$\\
7 & $\asp{as(4,q).}$\\
8 & $\asp{as(5,p).}$\\
9 & $\asp{as(5,q).}$\\
10 & $\asp{contr(p,t).}$\\
11 & $\asp{contr(q,r).}$\\
12 & $\asp{att(X,Y)\codeif contr(P,Q),root(X,Q),as(Y,P).}$\\
13 & $\asp{\hash show\; att/2.}$\\
\hline
\end{tabular}
\end{table}

Finally, to find the attack relations, we must introduce one last rule (line 12). This rule states that if an assumption is contrary to a root, then there is an attack from the argument containing the assumption towards the argument containing the root. By solving the program in Listing~\ref{lst:generateattacks} with clingo, we obtain one answer set: $\asp{\{att(4,1), att(4,3), att(4,5), att(5,2), att(5,4),}$ $\asp{att(5,5)\}}$. We now have an AAF representation of the original ABA framework. At this point we can apply the unified LAS task to learn the semantics of the ABA framework.

\section{Evaluation}
In this section, we evaluate the performance of our method by comparing it with ASPARTIX and a Deep Learning technique for learning argumentation semantics.
Comparison with ASPARTIX mainly concerns time required to compute extensions. When comparing to the Deep Learning approach we instead focus mainly on the dataset size needed to learn the AAF semantics. Note that, given that we learned semantics from small datasets, the time required for the training phase is constant and very small ($<10$ seconds for all considered semantics on our architecture); therefore, we are not going to discuss this in deeper detail. We ran the experiment on a MacBook Pro M1 2020 with 16GB of RAM. For reproducibility of results, the code used for benchmarking is available at \href{https://github.com/dasaro/ArgLAS}{\texttt{https://github.com/dasaro/ArgLAS}}.

\subsection{Time Performance for Computing AAF Extensions}
To evaluate time performance, we used the benchmark dataset from the ICCMA-23 competition \citep{iccma2023}.
In our evaluation we measured the time it takes to find one extension\footnote{We use the \texttt{clingo} flag \texttt{-n 1} to limit clingo to finding one extension only.}. For each argumentation framework, we construct answer set programs using our learned encoding and the ASPARTIX encoding for the given semantics. We record the time taken to complete the task for each argumentation framework. 
Figure \ref{fig1} shows average PAR-2 scores for admissible, complete, grounded, preferred and stable semantics on the ICCMA-23 dataset. The PAR-2 score is the index used to rank solvers in the ICCMA-23 competition, and it is defined for any specific instance as $2\cdot 1200$ if a threshold time limit ($1200$ seconds) is reached, and solving time otherwise.

\begin{figure}[t]
\centering
\includegraphics[width=\columnwidth]{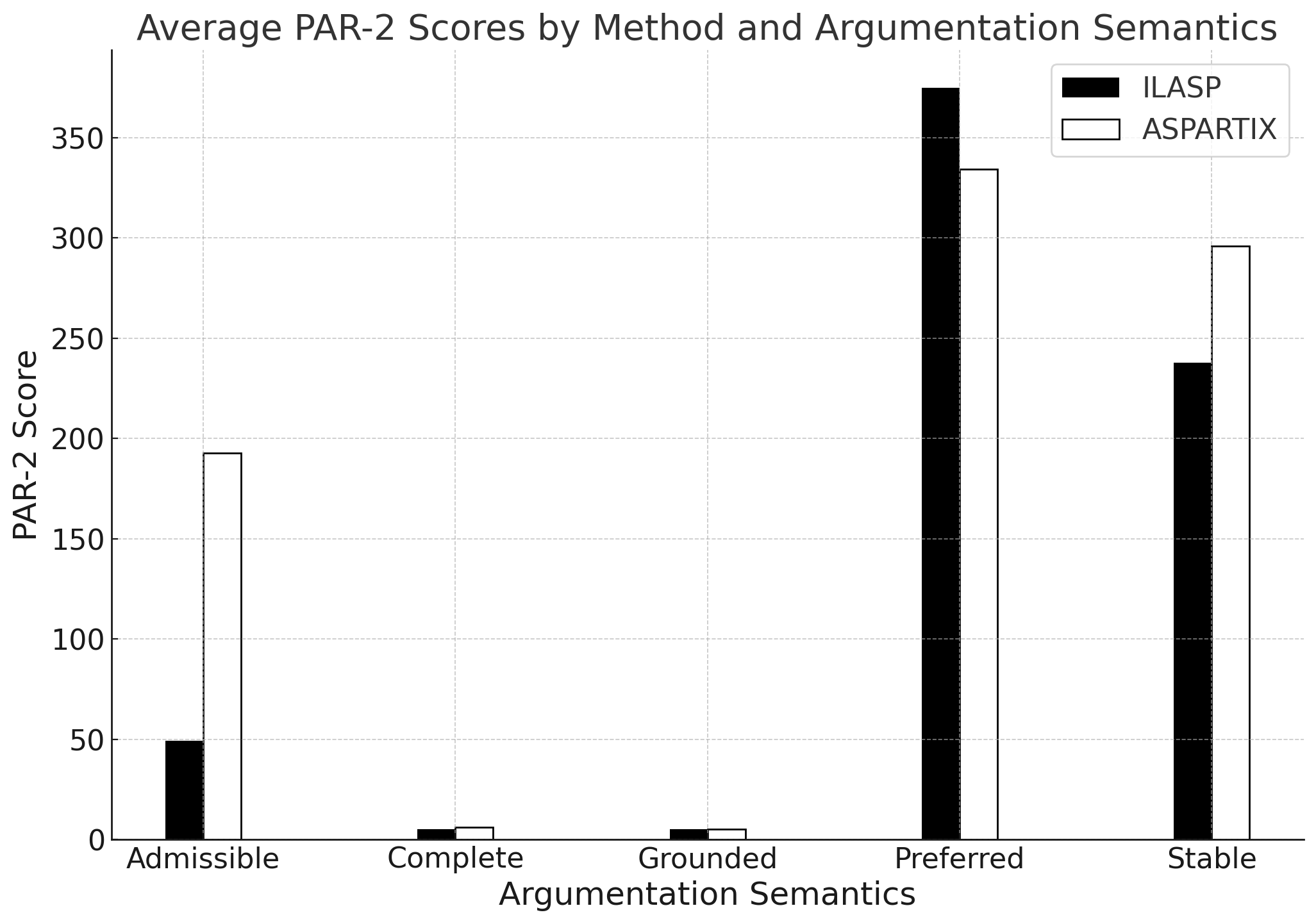}
\caption{Average PAR-2 scores for ASPARTIX and ILASP for different semantics on the ICCMA-23 dataset (the lower the better).}
\label{fig1}
\end{figure}

Our approach exhibits remarkable scalability compared to  ASPARTIX on the admissible and stable semantics. For the complete and grounded semantics, ASPATRIX and the ILASP-learned encodings do not show significant differences. On the other hand, for what concerns the preferred semantics, ASPARTIX shows better performance than ILASP-learned encodings.

\subsection{Comparison to Deep Learning methods}



Compared to the Deep Learning algorithm \citep{cb:20}, which was trained on one million examples and tested on a randomly generated set of $1000$ frameworks (each containing from $5$ to $25$ arguments), our method stands out, as it achieves perfect accuracy while not requiring as much data for the training phase.

Unlike our approach, which trivially achieves perfect accuracy for all the considered semantics ($MCC=1$, where $MCC$ is the Matthew Correlation Coefficient:\\
\begin{center}
\footnotesize$MCC=\frac{TP \cdot TN - FP \cdot FN}{\sqrt{(TP + FP)(TP + FN)(TN + FP)(TN + FN)}}$
\end{center}
for $TP$ the True Positive rate, etc.), the considered Deep Learning algorithm reaches $MCC=1$ only for the grounded semantics, and achieves near-perfect performance for the stable ($MCC=0.998$), preferred ($MCC=0.998$), and complete ($MCC=0.999$) semantics, while still needing larger amounts of training data compared to our proposed method. In fact, we use only $7$ examples to learn the admissible semantics, $8$ for both the stable and the complete semantics, $16$ for the preferred semantics, and $27$ for the grounded semantics, consistently achieving perfect accuracy ($MCC=1$). In contrast, when training the Deep Learning model on $30$ examples, which exceeds the number required by our method for any of the semantics, the MCC metric remains much lower than $1$. In this case, the maximum $MCC$ achieved by the Deep Learning algorithm is $0.39$ for the grounded semantics.

Table \ref{table2} summarizes MCC performance for each semantics after training the Deep Learning method on $30$ examples, for frameworks with $5$ to $25$ arguments. Despite training for $100$ epochs, extending the training time does not improve the outcome. The results consistently fall short of those achieved by our method.

\begin{table}
\centering
\begin{tabular}{|l|c|c|c|c|c|}
\hline
& Stable & Preferred & Complete & Grounded \\ \hline
MCC & 0.05386  & 0.30265    & 0.21286   & 0.38847   \\ \hline
\end{tabular}
\caption{MCC metric on the test set, when training the Deep Learning method in \citep{cb:20} on 30 examples for 100 epochs.}\label{table2}
\end{table}

Furthermore, the size of the frameworks we use for learning is small. For the preferred and the grounded semantics, the frameworks contain at most $5$ arguments, while for the stable, the admissible and the complete semantics at most $4$.

\subsection{Compactness and Interpretability}
Using ILASP, we have learned a streamlined and intelligible representation of each semantics. The rules we have learned are not only concise but are also transparent, making them easily explainable. In comparison, the ASPARTIX encodings contain more rules, especially for the grounded and the preferred semantics. For instance, while we employ $5$ rules for each of these two semantics in AAF, ASPARTIX's encodings use $13$ and $33$ rules, respectively. Additionally, the state-of-the-art Deep Learning algorithm we used in the experiments is not interpretable and does not achieve perfect accuracy even when training on very large datasets.

\section{Conclusion and Future Work}
This paper presents a novel approach to learning the acceptability semantics of argumentation frameworks, which relies on Learning from Answer Sets. We constructed a unified framework for learning the semantics of four argumentation frameworks. We proved the equivalence of the semantics learned by our $LAS_{arg}$ framework with the manually engineered ASPARTIX encodings for the stable, complete, and admissible semantics. In addition, empirical evaluations demonstrate that our method, while being able to learn from data, sometimes achieves better accuracy, data efficiency, and time performance when compared to other state-of-the-art methods.

The achievements delineated in this paper pave the way for multiple avenues of further research. First and foremost, we intend to explore the learning of domain-specific custom semantics. Recognizing that individual reasoning about the acceptability of arguments may vary widely in real-world contexts, our framework is designed to accommodate these unique perspectives, extending beyond the five known semantics considered within this study. Moreover, our ambitions extend to practical applications, where the learned encodings can be harnessed to ascertain accepted arguments from real-world dialogues. This can be achieved by integrating our method with other Machine Learning or Natural Language Processing tools dedicated to argument extraction from dialogues. These advancements may open up exciting possibilities for both theoretical exploration and practical utilization.

\bibliographystyle{kr}
\bibliography{kr-sample}

\end{document}